\documentclass[preprint]{elsarticle}

\usepackage{amsmath}
\usepackage{amsthm}
\newtheorem{theorem}{Theorem}[section]

\newtheorem{lemma}[theorem]{Lemma}
\newcommand{\vecp}{{\bf p}}
\newcommand{\onehalf}{{\textstyle \frac{1}{2}}}

\usepackage{array}
\newcommand{\PreserveBackslash}[1]{\let\temp=\\#1\let\\=\temp}
\newcolumntype{C}[1]{>{\PreserveBackslash\centering}p{#1}}
\newcolumntype{R}[1]{>{\PreserveBackslash\raggedleft}p{#1}}
\newcolumntype{L}[1]{>{\PreserveBackslash\raggedright}p{#1}}

\title{Pairwise coupling of convolutional neural networks for better explicability of classification systems}
\author[a,b]{Ondrej \v Such}
\address[a]{Mathematical Institute of Slovak Academy of Sciences, Bansk\'a Bystrica, Slovakia}
\address[b]{University of \v Zilina, \v Zilina, Slovakia}
\author[a]{Andrea Tinajov\'a}
\author[b]{Katar\'ina Bachrat\'a}
\author[b]{Peter Tar\'abek}

\begin{document}

\begin{frontmatter}

\begin{abstract}
	
We examine several aspects of explicability of a classification system built from neural networks. The first aspect is the pairwise explicability, which is the ability to provide the most accurate prediction when the range of possibilities is narrowed to just two. Next we consider explicability in development, which means ability to make incremental improvement in prediction accuracy based on observed deficiency of the system. Inherent stochasticity of neural network-based classifiers can be interpreted using likelihood randomness explicability. Finally, sureness explicability indicates confidence of the classifying system to make any prediction at all. 

These concepts are examined in the framework of pairwise coupling, which is a nontrainable meta-model that originated during development of support vector machines. Several methodologies are evaluated, of which the key one is shown to be the choice of the pairwise coupling method. We compare two methods -- the established Wu-Lin-Weng method with the recently proposed Bayes covariant method. Our experiments indicate that the Wu-Lin-Weng method gives more weight to a single  pairwise classifier, whereas the latter tries to balance information from the whole matrix of pairwise likelihoods. This translates into higher accuracy, and better sureness predictions for the Bayes covariant method.

Pairwise coupling methodology has its costs, especially in terms of the number of parameters (but not necessarily in terms of training costs). However, when additional explicability aspects beyond accuracy are desired in an application, the pairwise coupling models are a promising alternative to the established methodology.

\end{abstract}



\end{frontmatter}

\section{Introduction}

Deep convolutional neural networks \cite{DeepLearning} have emerged as a very powerful type of classification model  that is finding applications across diverse  applications.  A typical neural network model has many millions of parameters embedded in a deep multilayer structure making them very opaque classifiers, whose inner workings are poorly understood. The lack of explicability of deep neural networks is emerging as their critical deficiency. However, one should note that explicability is a multifaceted phenomenon. Let us consider several practical examples illustrating various distinct viewpoints at explicability. 

Deep neural networks are considered for use in scenarios that may result in fatalities. A severe malfunctioning of a system comprising a neural network, such as the Uber autonomous car fatality in Arizona, will have negative impact on the public's acceptance of AI technologies. Therefore, as well as for legal reasons, one would like to have a precise analysis, ideally with the identification of the root cause, why the misprediction occurred. We call this requirement \emph{post-hoc explicability}.

Clearly, it would be much better to have \emph{explicability in development}, when there is still an opportunity to change the classification system in order to prevent a possible malfunction  (mispredictions) during system deployment. However, one runs into a problem that complete retraining of a large network is often too costly, especially since there is no guarantee that it will bring the desired improvement. It would be desirable to perform only \emph{incremental improvement} of the system, so that most of the training effort is retained. Yet, in standard neural networks it is next to impossible to identify the subset of weights or a substructure of a large network that causes erroneous predictions. 

Consider next a physician specialist who is being assisted by a system employing a convolutional neural network in diagnosis based on patient data (e.g. X-ray, MRI, or histology). A convolutional neural network may arrive at a class prediction while the specialist may feel that  a different class is the correct one. Having narrowed the set of possibilities to a pair of classes, it would be helpful for the classification system to provide the most precise prediction possible given that there are only two classification outcomes under consideration. An improvement is not possible in classification systems that exhibit independence of irrelevant alternatives e.g. multinomial (softmax) regression. But if an improvement is possible, we say that a system exhibits \emph{pairwise explicability}.

Neural networks are affected by multiple randomness effects - from initialization through random dropout layers to random ordering of training samples in batches. To the end user, this \emph{prediction randomness} is rarely explained, mainly because it would be costly to train many networks to obtain meaningful measures of randomness for the predictions. We shall say that systems providing an explanation of stochasticity of predictions  have \emph{likelihood randomness explicability}.

Abstaining classifiers \cite{abstaining1} exhibit yet another take on explicability, which we call \emph{sureness explicability}. They are able to decide on one extra outcome - "based on my training data I don't feel confident to make a prediction", which is a very human-like behavior.

Given the complex nature of explicability and its varied benefits, one may expect that engineering explicability into classification systems entails costs and poses hurdles during model development. In our paper we examine these tradeoffs for a class of models constructed by pairwise coupling from convolutional neural networks. These models attack primarily the problem of pairwise explicability, but we will examine them also from the other viewpoints outlined here.

Organization of the paper is as follows. In Section 2 we will review principles of the pairwise coupling methodology. In Section 3 we will  provide  experimental details including the description of convolutional network architectures. In Section 4 we examine \emph{pairwise explicability}, by  evaluation of pairwise accuracy of networks trained in pairwise manner. In Section 5 we will evaluate multi-class accuracy of models built from networks trained in pairwise manner. In Section 6 we will illustrate \emph{explicability in development} by the concept of incremental improvement based on errors in the confusion matrix. In Section 7 we will illustrate \emph{likelihood randomness explicability} by constructing many classification models in order to estimate uncertainty of a prediction. In Section 8 we illustrate \emph{sureness explicability} i.e. how to use pairwise coupling models to detect images that are not consistent with the training set. Open research questions are discussed in the conclusion (Section 9).

\section{Review of pairwise coupling }

In this section we outline the pairwise coupling classification methodology proceeding from general definitions through historic motivation (SVM multi-class classification) to presenting state-of-the-art methods.

\subsection{Classification generalities}

A (hard) classification problem can be defined as search for a classifier function $f: X \rightarrow C$ mapping classified objects to a finite set $C$ of dependent categories (classes) $C_1, \ldots C_c$ \cite{isl}. When $c=2$ we speak of a \emph{binary} classifier.

Often, one solves this problem by solving the \emph{soft} classification problem first. This involves finding a posterior approximating predictor function $\tilde f : X \rightarrow [0,1]^{c}$, where $c$ is the number of classes in $C$ i.e. $C$ is the cardinality $|C|$ of the set $C$. If $\tilde f(x) = \left(\tilde f_1(x), \ldots, \tilde f_ c(x) \right)$ then the prediction is 
$$
f(x) = \arg\max_{i \leq c} \tilde f_i(x).
$$
The main reason for this detour is that it is convenient to optimize a smooth cost function 
of a parametric classification model by a gradient descent method. Such gradient search underlies many machine learning methods ranging from logistic regression to neural networks.

\subsection{Motivation for pairwise coupling}

A notable exception which did not fit the soft classification approach was the support vector machine model (SVM), a non-parametric classification technique very popular  since the late 20th century \cite{Vapnik1, Vapnik2}. The SVM model divides the feature space (or its higher-dimensional embedding via a kernel)  into two subsets by a hyperplane, which is found by solving a quadratic programming problem. The beautiful bisection idea of SVM posed a problem for multi-class classifications problems (when $c> 2$), because there was no obvious generalization of the quadratic programming problem to more than two classes. A methodology has been developed that entailed three major steps. 

The first step was adoption of one-on-one classification paradigm. This paradigm is characterized by requiring creation of all possible  pairwise SVM classifiers $f^{i,j}$ which are able to distinguish between the two classes  $C_i$ and $C_j$ only. There were two immediate advantages to doing so. Since the model required training only a portion of overall multiclass data, each pairwise classifier was easier to train. Moreover, each two-class dataset is more likely to be balanced, which would not be the case if one opted for one-vs-rest approach \cite{Mayoraz}.

The second step was converting hard classification model $f^{i,j}$ to a soft classification model $\tilde f^{i,j}$  by fitting a sigmoid function for each model. This was proposed in work of J. Platt \cite{PlattSVM}. A subtle point in his approach was adoption of uninformative prior on the labels, which avoids overfitting problems of logistic regression.

The third step is known in literature as \emph{pairwise coupling} approach \cite{HT, pairwisePatRec}, which converts the set of pairwise prediction provided by $\tilde f^{i,j}$ to final multi-class posterior prediction.

\subsection{Relationship between pairwise likelihoods and multi-class likelihoods}
To understand pairwise coupling's underlying principles, it is worthwhile to consider its ``reverse'' first. Suppose we are given a soft classifier which produces class posterior vector $p(x) = \bigl(p_1(x), \ldots, p_c(x)\bigr)$ for a sample $x$. Given such classifier and a pair of classes $i,j$ with $1\leq i,j \leq c$, one may construct a set of soft binary (i.e. two-class) classifiers $B_{ij}$, which we call  the \emph{IIA restrictions} of $p$. 

The naming is inspired by the axiom of independence of alternatives (IIA). In individual choice theory the IIA axiom  states that if an alternative $x$ is preferred from a set $T$, and $x$ is also an element of a subset $S$ of $T$, then $x$ should be also the preferred choice from $S$.

The \emph{IIA restriction} classifier quantifies this principle. Its likelihoods are such that the relative likelihood ratio of classes $i$ and $j$ is the same as in the presence of the rest of alternatives.  Therefore the IIA classifier $B_{ij}$ outputs the posterior vector 
\begin{align}
r_{ij}(x) = \Bigl(\dfrac{p_i(x)}{p_i(x) + p_j(x)}, \dfrac{p_j(x)}{p_i(x) + p_j(x)}
			\Bigr),
			\label{eq:iiabin}
\end{align}
which is the unique two-class probability distribution for which relative likelihood of classes is $p_i(x) / p_j(x)$.

Ratios in the equation \eqref{eq:iiabin} may produce singular results if both $p_i(x)$ and $p_j(x)$ are simultaneously zero.  We note that the singularity is avoided for a typical convolutional neural network, where the output of softmax layer cannot produce a zero posterior for any class.

Thus given multi-class prediction $p(x)$ we may construct the \emph{matrix of pairwise likelihoods}
\begin{align}
{\bf R}(x)= \begin{pmatrix}0 & r_{12}(x) & \ldots & r_{1c}(x) \\
            r_{21}(x) & 0 & \ldots &r_{2c}(x) \\
            \vdots &  \vdots & \ddots & \vdots \\
            r_{c1}(x) & r_{c-1,2}(x) & \ldots & 0
            \end{pmatrix}, \label{eq:rx}
\end{align}

In this paper we adopt the convention that on the diagonal of a matrix of pairwise likelihoods  there are always zeros. Then the matrix of pairwise likelihoods has all entries in the interval $[0,1]$ and satisfies 
\begin{align}
{\bf R}(x) + {\bf R}(x)' = \begin{pmatrix}
0 & 1 & \ldots & 1\\
1 & 0 & \ldots & 1 \\
1 & 1 & \ldots & 0  \label{eq:pairreq1}
\end{pmatrix}.
\end{align}

An important fact is that the mapping $\Theta: p(x) \mapsto {\bf R}(x)$ does not lose any information.

\begin{lemma}
	\label{lemma:column}
The nonlinear mapping $p(x) \mapsto {\bf R}(x)$ is injective on the set of nonvanishing posteriors. In fact, if $p_i(x)\not=0$ for all $i$ then it is possible to reconstruct $p(x)$ from any column, or any row of the pairwise likelihood matrix.
\end{lemma}
\begin{proof}
See e.g. \cite{SBT}, or \cite{PKPD}.  
\end{proof}

\subsection{Pairwise coupling methods}

\label{sec:pairwise-coupling-methods}

By a \emph{pairwise coupling method} we mean any method
mapping the set of nonnegative matrices ${\bf R} = (r_{ij})$ satisfying \eqref{eq:rx} to the set of probability distributions on $c$ classes.  We say that a pairwise coupling method is \emph{regular} if it 
 inverts the map $\Theta$ from the matrix from multi-class posteriors $p(x)$ to the correspoding pairwise likelihood matrix  ${\bf R}(x)$. 
  Thus if a regular pairwise-coupling method is  given a matrix of pairwise likelihoods constructed from a multi-class vector $\bf p$ by \eqref{eq:iiabin} and \eqref{eq:rx}, the  method should yield the original multi-class distribution.

The requirements on regular pairwise coupling are rather weak, because the mapping is prescribed only on $n-1$ dimensional subspace of $(n^2 -n) / 2$ dimensional parameter space of matrices satifying \eqref{eq:rx}. Therefore there exist many different regular coupling methods.

In our work we opted to use two regular pairwise-coupling methods: the method of Wu-Ling-Wen \cite{WLW}, and the Bayes covariant method \cite{SB}. The former is used in popular LIBSVM library \cite{LIBSVM}, and the latter has been proven to be a unique method satisfying additional  hypotheses\cite{SB}.

\subsubsection{Method of Wu-Ling-Weng}

The method  defines an optimization objective $\delta_2$:
\begin{align}
\delta_2 := \min_\vecp
\sum_{i=1}^k \sum_{j:\, j \not=i}(r_{ij}p_j - r_{ji}p_i)^2 \label{eq:wlw2}
\end{align}
From the definition it is immediately clear that the functional is nonnegative.  Moreover, it is zero on the image of map $\Theta: p(x) \mapsto {\bf R}(x)$ because there from \eqref{eq:iiabin} we have
\begin{align}
r_{ij}p_j = \dfrac{p_i p_j }{p_i + p_j} = r_{ji} p_i.
\end{align}
Optimizing $\delta_2$ is also numerically efficient. Since the functional is a quadratic function of $p_i$'s it is possible to reduce optimization to solving a set of linear equations.

\subsubsection{Bayes-covariant coupling}

We proposed this method in our work \cite{SB}. The underlying idea is geometric. Let us call the variety of pairwise likelihood matrices (i.e. the image of map $\Theta$) the Bradley-Terry manifold.  
The method starts by mapping the matrix of pairwise likelihoods coordinate-wise via
\begin{align}
\theta_{ij}: r_{ij} \mapsto \log \biggl(\dfrac{1}{r_{ij}} - 1\biggr). 
\label{eq:theta}
\end{align}
In the new coordinate space the Bradley-Terry manifold becomes a linear subspace, and the method is simply the orthogonal projection on the subspace. 

\subsubsection{Other pairwise coupling methods}

Let us briefly outline other coupling methods which have been proposed in the literature. 
In their comprehensive study of pairwise coupling  Hastie, Tibshirani \cite{HT} introduce a coupling method that optimizes a functional derived from Kullback-Leibler divergence. Work of Zahorian, Nossair \cite{ZahorianNossair} on classification of vowels using neural networks introduces another coupling method. Regular pairwise methods based on reverting map $\Theta$ in a columnwise manner have been proposed in \cite{PKPD} and \cite{SBT}. Also, Wu-Lin-Weng studied another coupling method based on a quadratic functional of posteriors \cite{WLW}.

\subsection{Numerical stability of coupling methods}
\label{sec:numerical-instability}
Some coupling methods have numerically unstable behavior near the boundary of the space of possible pairwise likelihood matrices. For instance, the Bayes covariant coupling suffers from this problem, since the mapping $\theta_{ij}$ in \eqref{eq:theta} is singular at the limit points zero and one.

In our work \cite{SB} we have proposed two ways to deal with such instability:
\begin{itemize}
\item start by choosing a small threshold $\onehalf \gg \tau > 0$ and then force individual pairwise likelihoods $r_{ij}$ to lie in the interval $[\tau, 1 - \tau]$ by replacing them with $\min(1-\tau, \max(\tau, r_{ij}))$ if necessary, or 
\item choose a small threshold $\onehalf \gg \rho > 0 $, and remove from consideration any class $c$ for which there is $c'$ such that $r_{cc'} < \rho$,  i.e. remove all rows and columns from the pairwise likelihood matrix that correspond to such classes. Then apply the coupling method to the possibly smaller matrix of pairwise likelihoods. Finally, convert the posterior probability distribution to the full set of classes, for instance by extending with zero.
\end{itemize}

\section{Methods}
\label{sec:methods}

The  models built by pairwise coupling need to be built from binary classifiers. In this section we describe the dataset used, the three classes of convolutional neural networks employed as binary classifiers (micro-models, mini-models and macro-models), as well as the baseline multi-class networks which we used for comparison. 

\subsection{Dataset}

We have used Fashion MNIST (FMNIST) dataset \cite{FashionMNIST}. It has been suggested \cite{FashionMNIST} to be a better starting point for examining of computer vision classification methods compared to historically more popular MNIST dataset\cite{MNIST}. Moreover, the small size of the dataset resulted in short training times and thus in low environmental impact of our experiments, an increasingly important societal consideration \cite{EnergyAI}. The overall training and test data sizes (60000 and 10000 respectively) are identical to MNIST.  There are 10 classes as shown in Table \ref{tab:FMNIST}, evenly distributed in the dataset.

\begin{table}
	\centering
\begin{tabular}{lc}
	\bf class & \bf description\\
	\hline
	0 & T-shirt/top \\
	1 & trouser \\
	2 & pullover \\
	3 & dress \\
	4 & coat \\
	5 & sandal \\
	6 & shirt \\
	7 & sneaker \\
	8 & bag \\
	9 & ankle boot 
\end{tabular}
\caption{Classes in Fashion MNIST dataset.}
\label{tab:FMNIST}
\end{table}

\subsection{Baseline networks}

To create our baseline model (model $F$ in Figure \ref{fig:arch}) the Keras example CNN original designed for MNIST dataset (model $M$ in Figure \ref{fig:arch}). The only structural difference is that we replaced dropout layers with batch normalization layers. There are two reasons for this adaptation:
\begin{itemize}
	\item The optimal probability of dropout may vary depending on the task. The values chosen for MNIST task may not be suitable for Fashion MNIST task, and in this work we did not plan to do any hyperparameter tuning.
	\item In the most recent survey on dropout methods \cite{DropoutSurvey} it is stated that dropout is not necessary for convolutional networks, and batch normalization is sufficient \cite{BatchNormalization}, \cite{Dropout1}.
\end{itemize} 

We trained altogether 32 instances $F_0, \ldots, F_{31}$ of networks based on $F$ architecture, whose weights were later used in initialization of the binary classifiers.

\begin{figure}
	\centering
	\includegraphics[width = 0.7\textwidth]{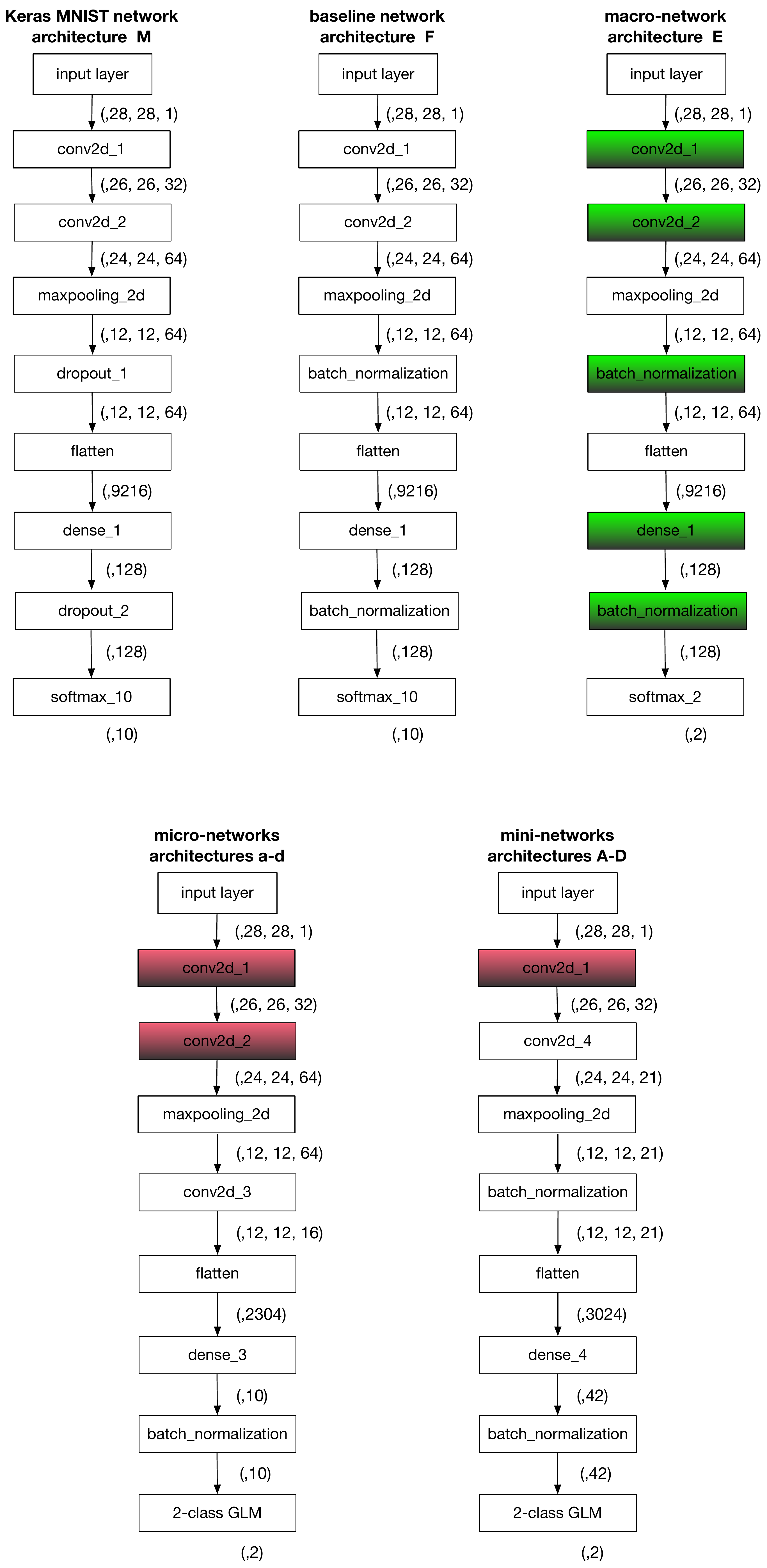}
	\caption{Schemas of feedforward CNN networks used for experiments, as well as the prototype network architecture from Keras library. Tensor shapes below each box indicate the layer's output shape. The red color indicates layers' weights was copied  from the baseline model with $F$ architecture and fixed throughout training. The green color indicates that the layers' weight was copied from the baseline model with $F$ architecture, but not fixed during training. The final two class general linear model (GLM) layer in architectures $a-d, A-D$ was either logistic regression or the GLM model with complementary log-log link function. }
	\label{fig:arch}
\end{figure}

\subsection{Architectures for binary models}

We have examined 9 different feed-forward  architectures for binary models, which we describe in this section. Roughly in increasing size they were micro-models (Section \ref{sec:micro}, see also architectures $a-d$ in Figure \ref{fig:arch}), mini-models (Section \ref{sec:mini}, architectures $A-D$ in Figure \ref{fig:arch}) and macro-models (Section \ref{sec:macro}, the architecture $E$ in Figure \ref{fig:arch}). For all of them we trained 32 different networks, each partially initialized with the weights of the corresponding baseline network $F_i$ ($i\leq 32$).

\subsubsection{Binary micro-models}

\label{sec:micro}

In Section \ref{sec:multi-class} we will compare models built using pairwise coupling methods with the baseline model. It is desirable to use models with the same number of training parameters. Since a complete pairwise-coupled model requires  45 binary classifiers, individual binary classifiers have to be rather small. Namely, since the model $G$ of the network has 1 200 650 parameters, individual binary classifiers should have $\approx 1200650 / 45 \approx 26681$ parameters. It is quite challenging to get good performance with a convolutional network that small. A natural first step is to reduce the number of neurons (convolutions) in individual layers. However,  based on our previous experience with share-none architectures \cite{SKT}, we felt that two additional alterations were needed:
\begin{itemize}
	\item Reducing the number of 2D maps by using $1\times 1$ convolution.
	\item Sharing weights among individual pairwise classifiers. Thus the first two layers of the binary classifiers had all identical weights copied from the corresponding baseline network. Those weights were not trainable.
\end{itemize}

The difference among the four variants of the micro-networks were as follows.

Networks $a$ and $b$ were trained with the softmax layers as the final layer. Networks $c$ and $d$ used the generalized linear model (GLM) model with complementary log-log link function, which is  considered more suitable in cases of non-symmetric distributions \cite{cloglog}. 

Networks $a$ and $c$ employed binary encoding of the dependent variable. Encoding of dependent variable in networks $b$ and $d$ was altered to $\epsilon$ instead of $0$ and $1-\epsilon$ instead of $1$ with $\epsilon = 1 / ( \textrm{size of the training set})$ in an attempt to:
\begin{itemize}
	\item mimic use of J. Platt's uninformative prior in SVM case, and 
	\item avoid numerical instabilities in pairwise-coupling, as explained in Section \ref{sec:numerical-instability}.
\end{itemize}

\subsubsection{Binary mini-models}

\label{sec:mini}

A complete model built using pairwise coupling from binary micro-models has the same number of parameters as the baseline model. However, it is much easier to train, since each pairwise training dataset has $c/2 = 5$ times less data in it compared to the full Fashion MNIST dataset. We can roughly estimate that the complexity of training one epoch of a network is proportional to the product of the number of weights and the number of training samples. Since each pairwise micro-model has  $~ c(c-1)/2$ less parameters, the complete arithmetical complexity to run the same number of epochs for all $c(c-1)/2$ networks is about $c/2 = 5$ times lower compared to the baseline model. 

We therefore investigated also pairwise coupling models, termed \emph{mini-models} that have approximately the same total arithmetical training complexity to the baseline model, although they have more parameters. The underlying binary mini-networks have $c/2$ times more parameters than micro-networks, or equivalently, $(c-1)$ times less parameters than the baseline networks.  Their architecture is shown in Figure \ref{fig:arch}.

Analogously to the case of micro-models we used 4 flavors of this architecture. The models $A,B$ used softmax as the final layer, whereas the models $C,D$ used complementary log-log GLM model. The models $A,C$ used standard binary encoding of the class variable, whereas the models $B,D$ used the encoding as for the models $b,d$.

\subsubsection{Binary macro-models}

\label{sec:macro}

We also included models of type $E$, shown in Figure \ref{fig:arch}, whose architecture is identical to the baseline model $F$ except for the final layer, which is not 10-class softmax layer, but just 2-class softmax layer. The models are initialized with the weights of the corresponding $F$ model, but all weights are left trainable. 



\subsection{Other  details}
Multi-class networks were trained using batch size 128, whereas for  binary models we decreased the batch size to 32. The number of epochs was 12 in all cases. 

In all cases we used AdaDelta stochastic gradient optimizer with standard settings and cross-entropy as the optimization criterion.

Finally, Table \ref{tab:arch-sizes} shows the parameter counts for all networks used.

\begin{table}
	\centering
	\begin{tabular}{L{1.5cm}C{2cm}R{2.3cm}R{2.3cm}R{2.3cm}}
		Architecture & Number of prediction classes & Total parameters & Trainable parameters & Non-trainable parameters\\
		\hline
		a,b & 2 & 42,968 & 24,132 & 18,836 \\
		c,d & 2 & 42,957 & 24,121 & 18,836 \\
		A,B & 2 &  133,777 & 133,331 & 446 \\
		C,D & 2 & 133,734 & 133,288 & 446 \\
		E & 2 & 1,199,618 & 1,199,234 & 384 \\
		\hline
		F & 10 & 1,200,650 & 1,200,266 & 384 \\		
	\end{tabular}
	\caption{Sizes of network architectures}
	\label{tab:arch-sizes}
\end{table}

\section{Evaluation of pairwise accuracy}

In Section 1 we introduced the concept of pairwise explicability. This concept refers to the situation, when we are confident that only two possible predictions are possible and we desire the best possible prediction by a convolutional neural network. Of course  one may obtain a two-class prediction by taking IIA restriction of a multi-class classifier. But is it possible to do better by training specialized binary networks on only two classes?

In this section we present the answers to this question for the binary architectures described in Section \ref{sec:methods}.


\subsection{Influence of architecture}

\begin{figure}
	\includegraphics[width=0.95\textwidth]{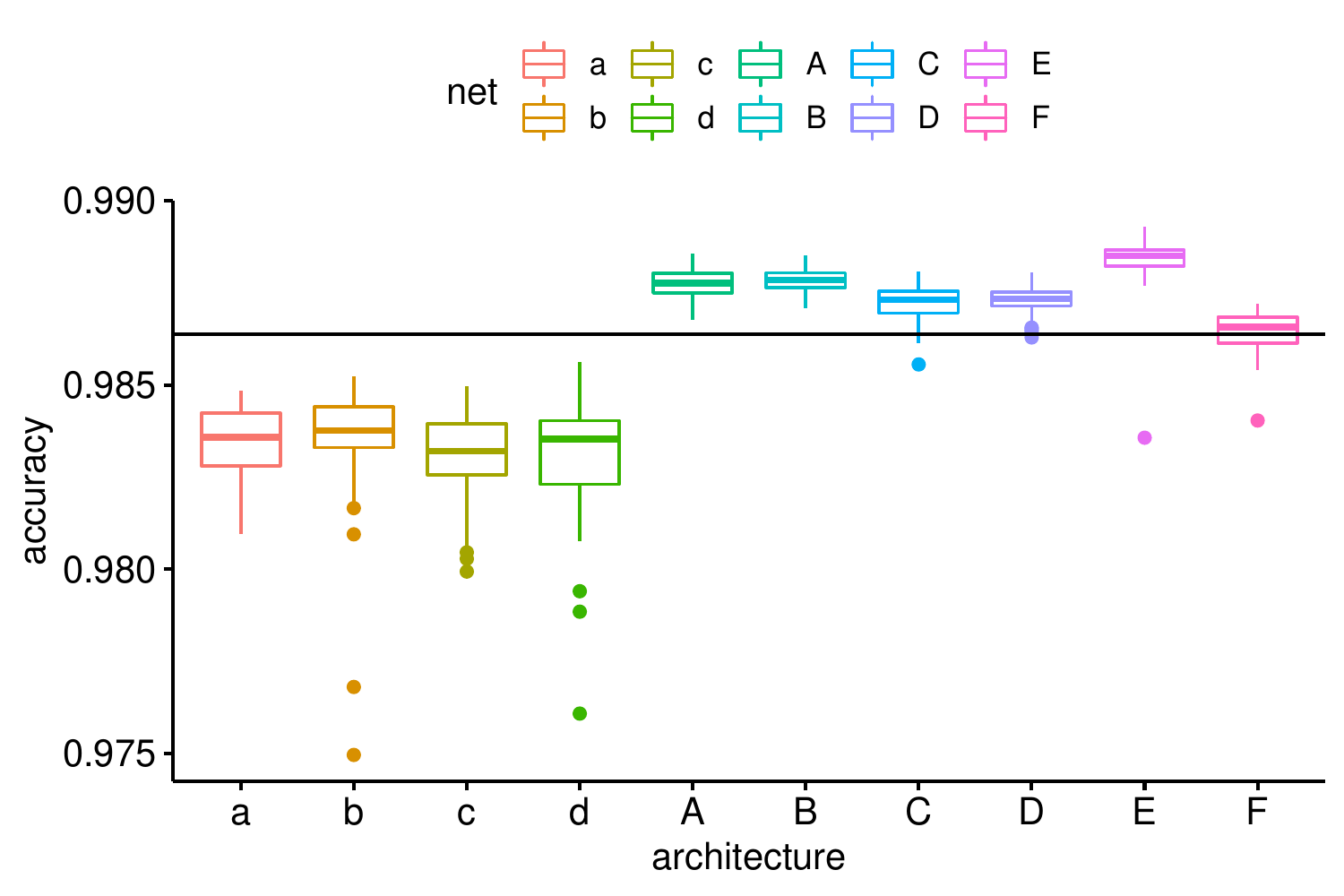}
	\caption{Average pairwise accuracy for architectures $a-d$, $A-F$ over 32 trials. The horizontal line indicates the mean pairwise accuracy for the architecture $F$.}
	\label{fig:anova1}
\end{figure}

In Figure \ref{fig:anova1} we have plotted boxplots for average pairwise performance over 32 trainings of each architecture.

From the figure we see the obvious trend that with increasing size, the convolutional networks perform better. The key point is that all mini-architectures $A-D$ (as well as the macro-architecture $E$) perform better than ()the IIA restrictions of) the standard multi-class network (architecture $F$). On the other hand all micro-networks of types $a-d$ perform worse than the multi-class network.

%
%
%

\subsection{Detailed performance by pairs of classes}

We also plotted the pairwise accuracy for each pair of classes in Figure \ref{fig:pair-detail}. From the figure it is clear that only a handful of pairs of classes show uneven performance among architectures:
\begin{itemize}
	\item coat/shirt
	\item dress/coat
	\item dress/shirt
	\item pullover/coat
	\item pullover/shirt
	\item t-shirt/shirt
\end{itemize}
Again, as observed above, it is the size of the architecture that is the primary factor affecting the pairwise accuracy. Moreover, the results strongly suggest that not all decision boundaries are equally difficult to find, and thus in multi-class pairwise-coupled models varying learning capacity (i.e. the number of trainable parameters) is likely needed to construct an optimal two-class classifier.

\begin{figure}
	\includegraphics[width=\textwidth]{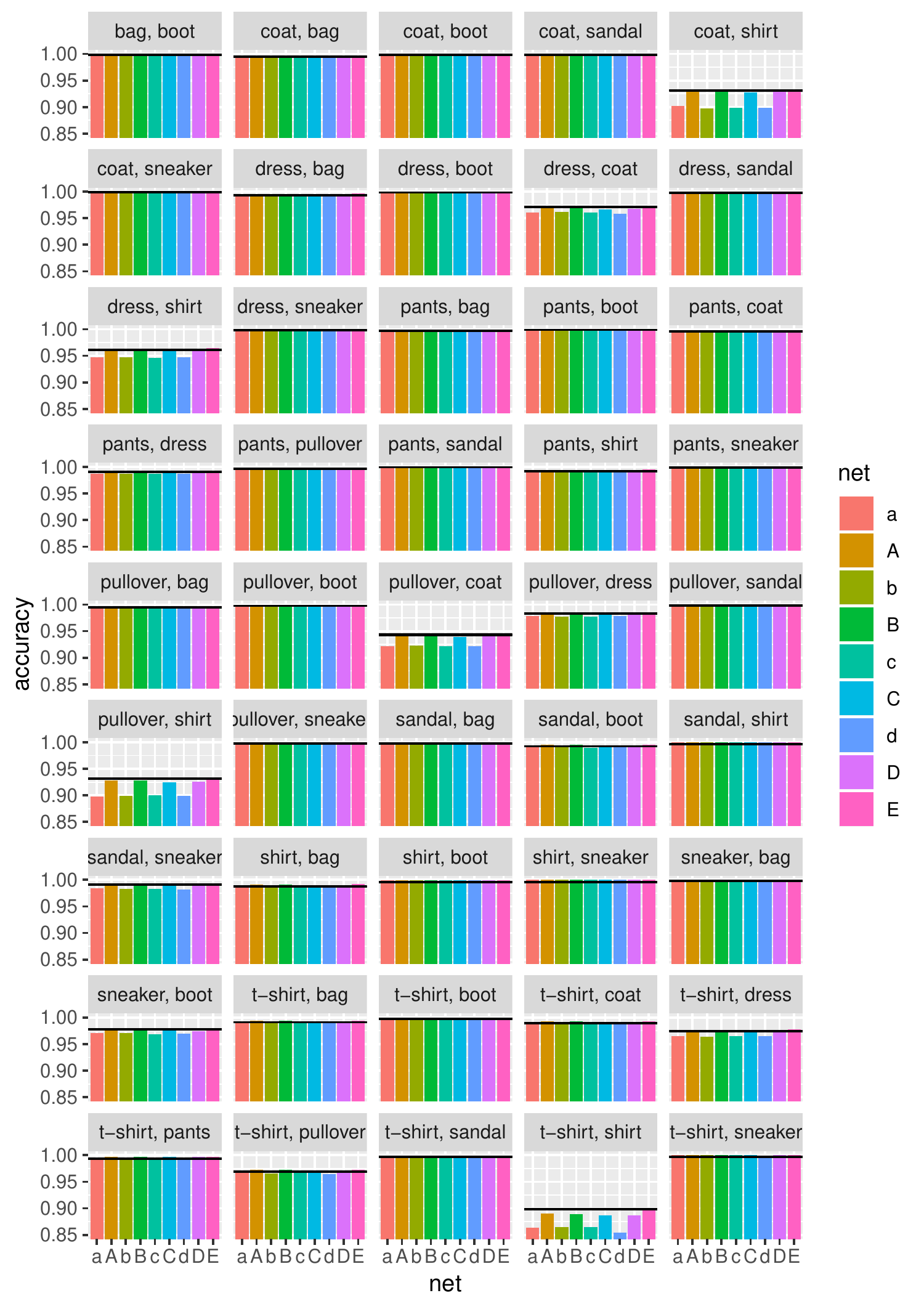}
	\caption{Average pairwise accuracy of individual architectures for all possible pairs of classes. The black horizontal line indicates mean accuracy of architecture $F$.} 
	\label{fig:pair-detail}
\end{figure}


\section{Multi-class accuracy of models built by pairwise coupling}
\label{sec:multi-class}

In the previous section we have concentrated on the question whether pairwise accuracy can be improved by training network with two classes only. We found that for mini-networks the answer is yes. In this section we investigate whether this improvement translates to better multi-class accuracy in pairwise coupling models. There is even the possibility that despite the worse pairwise performance of micro-networks compared with the baseline models, using pairwise coupling may repair their pairwise mispredictions and the multi-class models built from micro-networks using pairwise coupling could be more accurate that the baseline models. This effect, which we term \emph{coupling recovery} was examined on synthetic data sets in works \cite{HT} and \cite{WLW}.

\subsection{Preliminaries}
In order to build the multiclass models from pairwise ones we need to take into account a couple of additional factors. The first is that in addition to architecture, one needs to choose a coupling method. As indicated in Section \ref{sec:pairwise-coupling-methods} we are going to use two coupling methods, the canonical method of Wu-Ling-Weng and the Bayes covariant method. Additionally, the Bayes covariant method is not numerically stable, so that one needs to choose a way to avoid numerical singularities. We will use the first method outlined in \ref{sec:numerical-instability}. In order to do that, one needs to choose a threshold $ \onehalf \gg \tau > 0$, and we opted for the value $\tau = 10^{-3}$. We illustrate suitability of this value by plotting the accuracy of predictions for four different pairwise-coupling models based on the $A$ network architecture, as shown in Figure \ref{fig:threshold}.

\begin{figure}
	\centering
	\includegraphics[width=\textwidth]{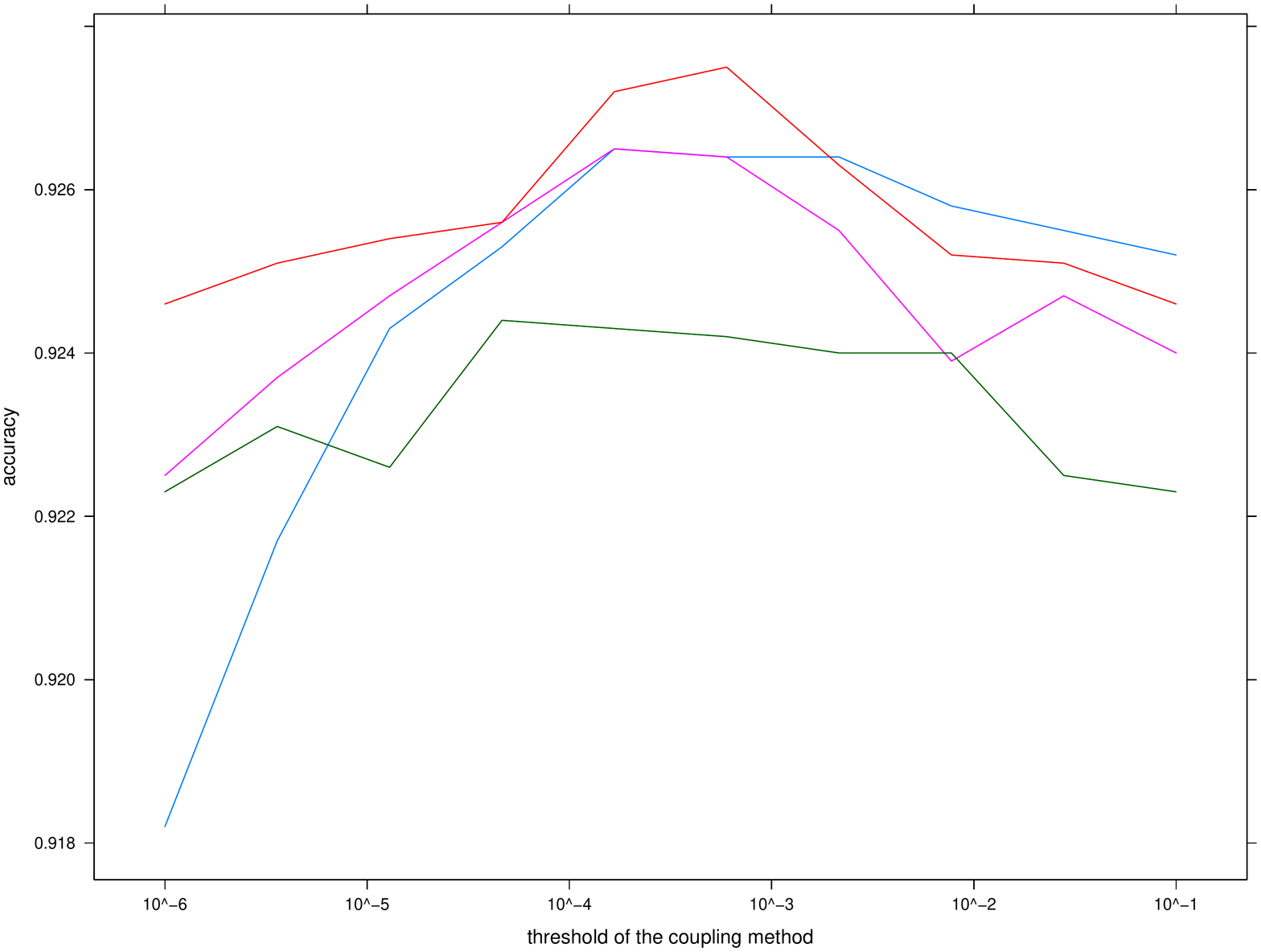}
	\caption{The accuracies of four different pairwise coupling models built from $A$ architecture and using the Bayes covariant coupling method with the given threshold $\tau$.}
	\label{fig:threshold}
\end{figure}

\subsection{Statistical analysis}

Let us start by calculating the average multi-class accuracies for each pair of architecture and the coupling method, as shown in Table \ref{tab:multi-class}.

\begin{table}
\begin{tabular}{ccc}
	network & Wu-Lin-Weng method & Bayes covariant method \\
	\hline
	a & 0.902 & 0.907 \\
	b & 0.904 & 0.909 \\
	c & 0.903 & 0.905 \\
	d & 0.901 & 0.903 \\
	\hline
	A & 0.923 & 0.926 \\
	B & 0.924 & 0.926 \\
	C & 0.921 & 0.923 \\
	D & 0.922 & 0.924 \\
	\hline
	E & 0.928 & 0.929 
\end{tabular}
\caption{Average multi-class accuracies of the models built from pairwise classifiers.}
\label{tab:multi-class}
\end{table}

The mean multi-class accuracy of the baseline models is 0.923. which shows that only some mini-models may outperform the baseline. We evaluated paired $t$-tests and obtained the results shown in Table \ref{tab:ttest1}.

\begin{table}
		\begin{tabular}{cccc}
			pair & coupling method & $p$-value &  significance\\
			\hline
			A - F & Wu-Lin-Weng & 0.8839\\
			B - F & Wu-Lin-Weng  & 0.4295 \\
			A - F &  Bayes covariant method & $0.0001$ & *** \\
			B - F & Bayes covariant method  & $9.2 \times 10^{-5}$& ***\\
			C - F & Bayes covariant method & $0.3767 $\\
			D - F & Bayes covariant method & $0.7246$ \\
		\end{tabular}
	\caption{Evaluation of $t$-tests comparing performance of a mini-model with the baseline multi-class model.}
	\label{tab:ttest1}
\end{table}

We conclude that  mini-models $A$ and $B$ outperform in multi-class setting the baseline networks when the Bayes covariant pairwise coupling method is used.

\section{Incremental improvement}

In this section we investigate the possibility of correcting a poorly performing multi-class convolutional network. A standard way to understand the poor performance of a multi-class classifier is to construct the confusion matrix of predictions. For example we trained a baseline classifier $F_0$ which had confusion matrix as shown in Table \ref{tab:conf1}.

\begin{table}
\begin{tabular}{r|rrrrrrrrrr}
	& 0 & 1 & 2 & 3 & 4 & 5 & 6 & 7 & 8 & 9 \\
	\hline
	0 & 881 & 2 & 12 & 9 &  6 & 2 & 86 & 0 & 2 & 0 \\
	1 & 3 & 987 & 1 & 4 & 1 & 0  & 3 & 0 & 1 & 0 \\
	2 & 24 & 2 & 872 & 6 & 45 & 1 & 49 & 0 & 1 & 0 \\
	3 & 22 & 6 & 8 & 896 & 21 & 0 & 46 & 0 & 0 & 1 \\
	4 & 3 & 0 & 30 & 22 & 886 & 0 & 58 & 0 & 1 & 0 \\
	5 & 0 & 0 & 0  & 0  & 0   & 987 & 0 & 8 & 0  & 5 \\
	6 & 91 & 0 & 43 & 21 & 58 & 0 & 781 & 0 & 6 & 0 \\
	7 & 0 & 0 & 0 & 0 & 0 & 10 & 0 & 977 & 0 & 13 \\
	8 & 3 & 0 & 1 & 5 & 1 & 2 & 4 & 3 & 980 & 1 \\
	9 & 0 & 0 & 0 & 0 & 1 & 6 & 0 & 37 & 1 & 955\\
	\hline
	\end{tabular}
\caption{The confusion matrix of the baseline classifier $F_0$ of type $F$.}
\label{tab:conf1}
\end{table}

From the table it is  clear that the network makes most errors confusing classes 0 and 6. We may try improving on classification of an image $x$ by applying a pairwise coupling method to the  pairwise likelihood matrix $R'(x) = (r'_{ij}(x)) $ constructed as follows. Let $\bigl(p_0(x), \ldots, p_9(x)\bigr)$ be the prediction of $F_0$ and let $\bigl(q_0(x), q_6(x)\bigr)$ be the prediction of a binary classifier $Q$ trained to distinguish classes 0 and 6. Then we replace  two entries of the pairwise likelihood matrix for $F_0$ by values provided by $Q$:
\begin{align}
r'_{ij}(x) = \begin{cases}
q_0(x)\qquad \textrm{if $i = 0$ and $j = 6$} \\
q_6(x)\qquad \textrm{if $i = 6$ and $j = 0$} \\
0\quad\qquad\,\,\,\,\, \textrm{if $i = j$}\\
\dfrac{p_i(x)}{p_i(x) + p_j(x)} \quad\textrm{otherwise}
\end{cases}
\label{eq:partial-correction}
\end{align}

We have illustrated the results of such partial correction on Figure \ref{fig:fixes-lm}.

\begin{figure}[!ht]
	\centering
	\includegraphics[width=0.7\textwidth]{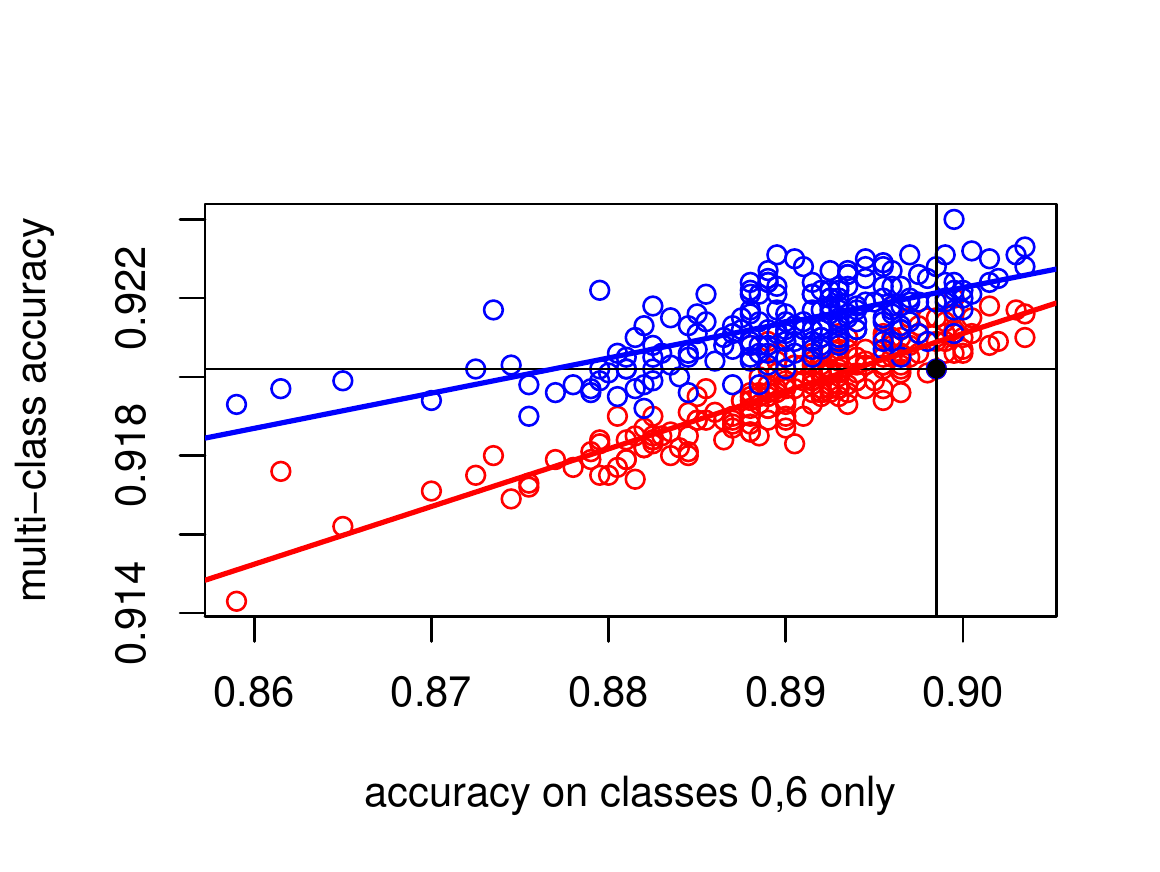}
	\caption{Plot of the multi-class versus the pairwise accuracy of 200 instances of partial corrections given by \eqref{eq:partial-correction}. We used mini-networks of type $A$ as the binary classifiers. Red color corresponds to results for pairwise coupling with the Wu-Lin-Weng method, whereas the blue color is for the Bayes covariant pairwise coupling. The black dot corresponds to the datum for $F_0$ model. Linear regression lines were added.}
	\label{fig:fixes-lm}
\end{figure}

We can conclude that even when  $Q$ has inferior pairwise accuracy to $F_0$, the coupling methods are able to increase the multi-class accuracy. The  Bayes covariant method (red regression line in Figure \ref{fig:fixes-lm}) is slighly better than the method of Wu-Lin-Weng (blue regression line) over the range of pairwise accuracies afforded by mini-networks of type $A$. However, the latter is more efficient in converting an increase of the pairwise accuracy to an increase of the multiclass accuracy. In fact, linear regression model for the method of Wu-Lin-Weng has $R^2 = 0.81$, whereas $R^2 = 0.46$ for the Bayes covariant method.

\section{Likelihood randomness explicability}

\label{sec:randomness-explicability}

Partial correction achieved in the previous section was enabled by inherent modularity of the pairwise coupling models. In this section we examine another application of this modularity.

Neural networks are inherently random algorithms. Their predictions will vary if they are repeatedly trained. This fact is obvious to specialists, but rarely conveyed to the end user. The primary obstacle is training cost. If it takes weeks to train a single deep neural network, it is utterly impractical to train 100 different  networks.  
We will illustrate on an example from Fashion MNIST that this obstacle is easily overcome in pairwise coupling  models. 


The image \#142 (counted from 1) in the test set is class 0 (t-shirt/top). It is shown in Figure \ref{fig:142}.

\begin{figure}[!ht]
	\centering
	\includegraphics[width=0.3\textwidth]{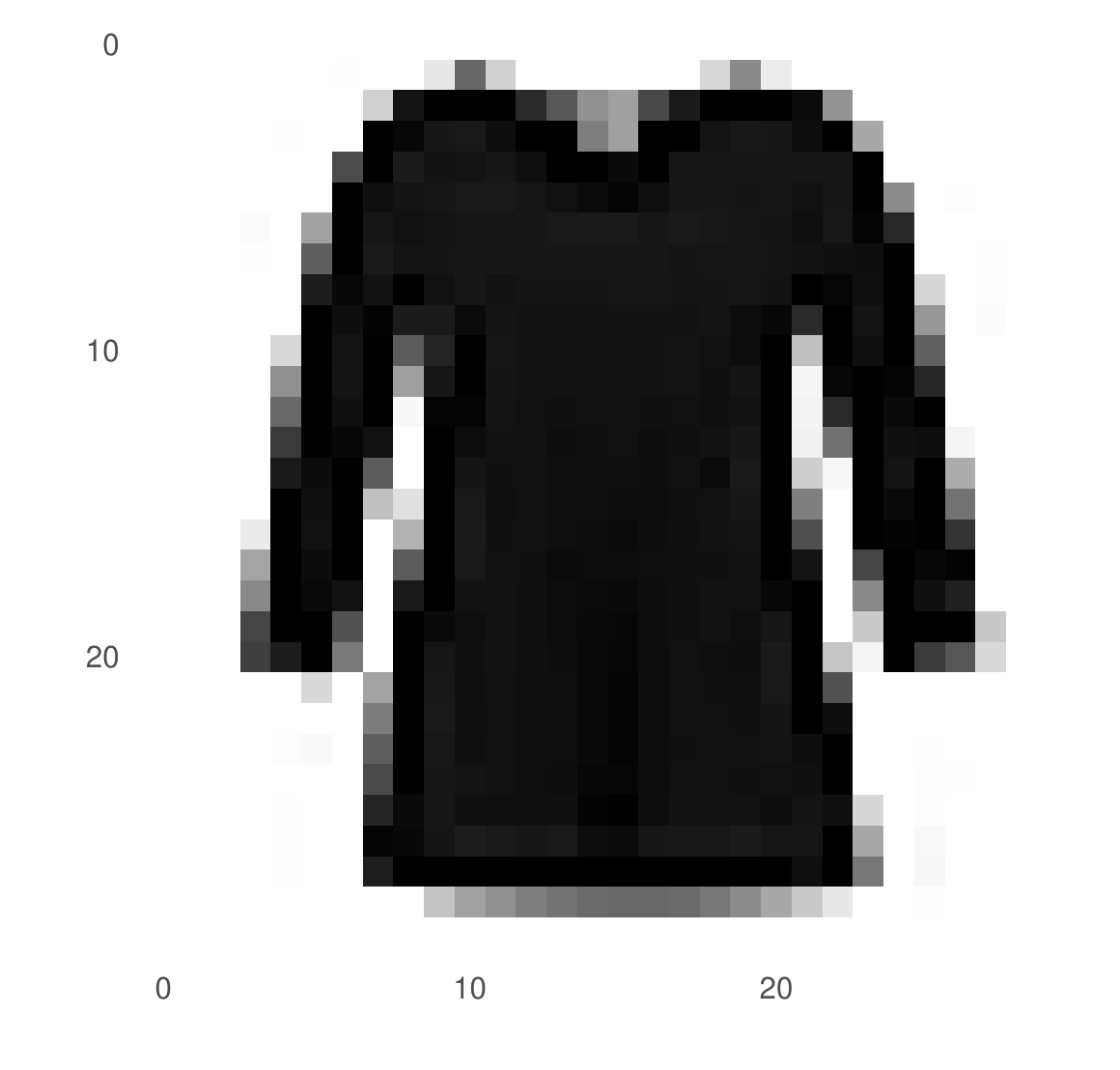}
	\caption{Image \#142 from the Fashion MNIST test data set.}
	\label{fig:142}
\end{figure}

This image is correctly predicted to be the class 0 by the network $A_1$, but incorrectly predicted to be the class 6 (shirt) by the network $A_2$. Both of them are quite sure, giving more than 95\% likelihood to their prediction as seen in Table \ref{tab:net2}.

\begin{table}[ht]
	\centering
	\begin{tabular}{rrrrrr}
		\hline
		network	& class 0 & class 1 & class 2 & class 3 & class 4  \\ 
		\hline
		$A_1$ & 9.53e-01 & 2.04e-08 & 1.13e-02 & 1.55e-05 & 1.53e-08 \\
		$A_2$ & 5.58e-08 & 3.59e-02 & 9.41e-06 & 5.06e-10 & 4.80e-11 \\ 
		\hline
		network & class 5 & class 6 & class 7 & class 8 & class 9 \\
		\hline
		$A_1$ & 1.63e-02 & 3.16e-11 & 4.69e-04 & 9.91e-04 & 4.83e-08 \\
		$A_2$ & 1.63e-09 & 9.82e-01 & 6.88e-07 & 2.32e-09 & 6.27e-09 \\ 
		\hline
	\end{tabular}
	\caption{Predicted multi-class likelihoods  for test image \# 142 by the networks $A_1$ and $A_2$}.
	\label{tab:net2}
\end{table}

Pairwise coupling approach allows one to create a vast number of new classifiers (more precisely $2^{45}$) by bootstrapping pairwise predictions from either of the two pairwise likelihood matrices corresponding to predictions of networks $A_1$ and $A_2$ and then applying a pairwise coupling method. We have plotted 100 samples for both pairwise coupling methods in Figure \ref{fig:100}. 

\begin{figure}
	\includegraphics[width=0.45\textwidth]{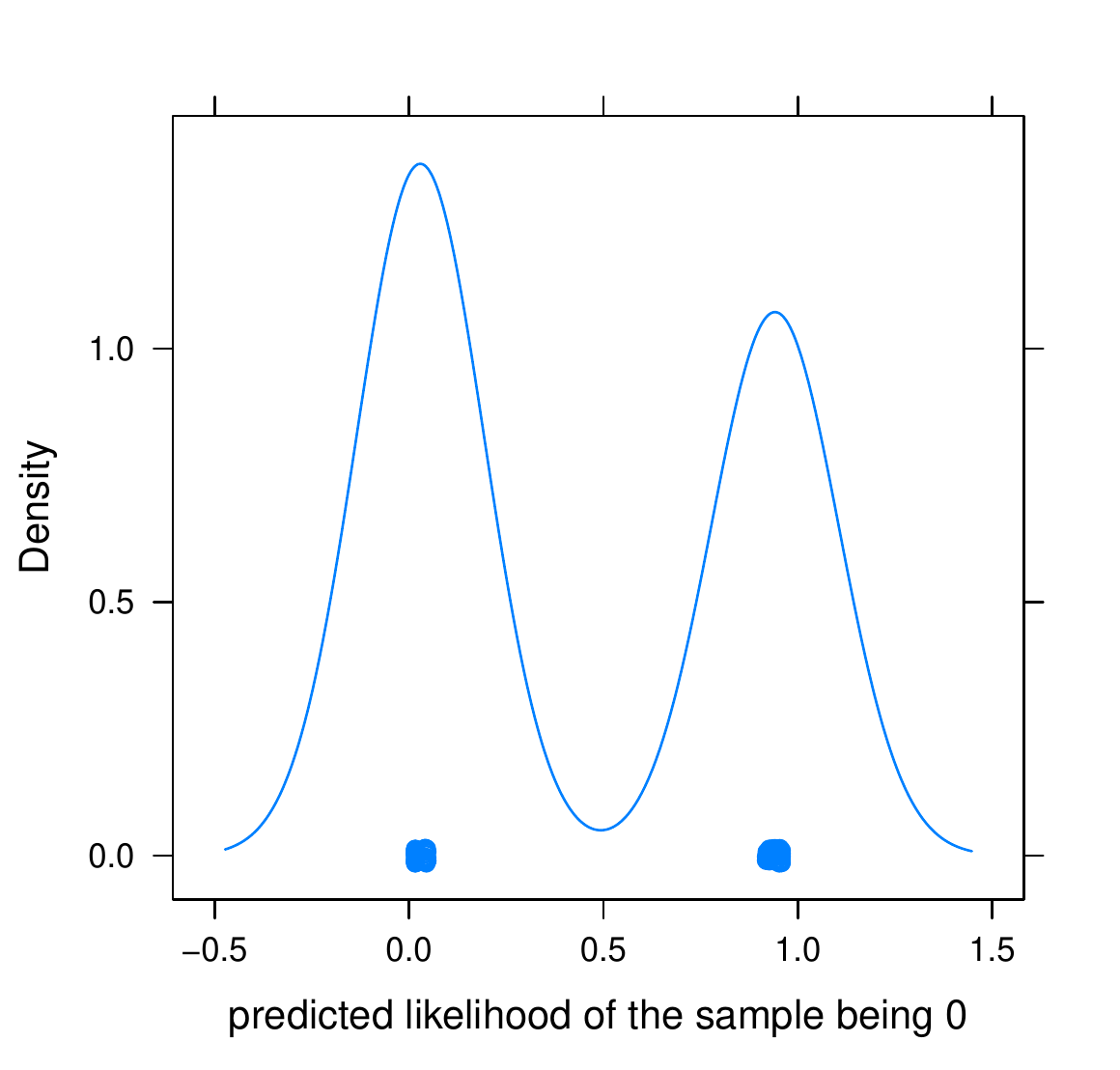}
	\includegraphics[width=0.45\textwidth]{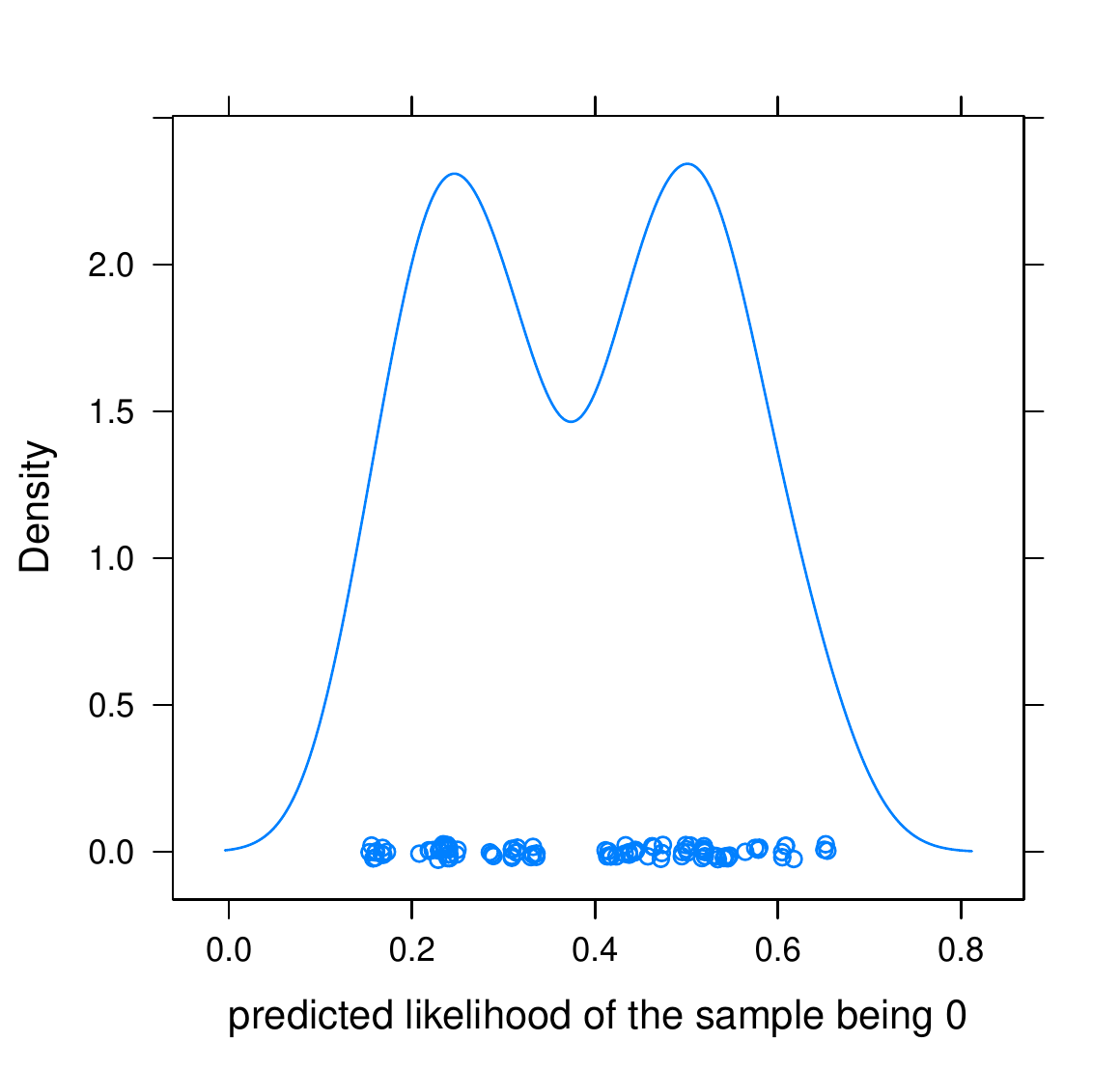}
	\caption{Predicted likelihoods of image \#142 for 100 randomly recombined pairwise-classifiers from networks $A_1$ and $A_2$. Left image: using the Wu-Lin-Weng method, right image: using the Bayes covariant method.}
	\label{fig:100}
\end{figure}

The plots show that there is a significant difference between the method of Wu-Lin-Weng and the Bayes covariant method. The former is strongly clustered in two clusters near 0 and 1, whereas the latter is much more uniformly distributed from 0.1 to 0.9. The natural guess that the binary classifier for the pair 0/6 has the crucial influence on the multi-class posteriors is confirmed by plotting conditional density plots in Figure \ref{fig:wu23}.

\begin{figure}[ht]
	\includegraphics[width=0.45\textwidth]{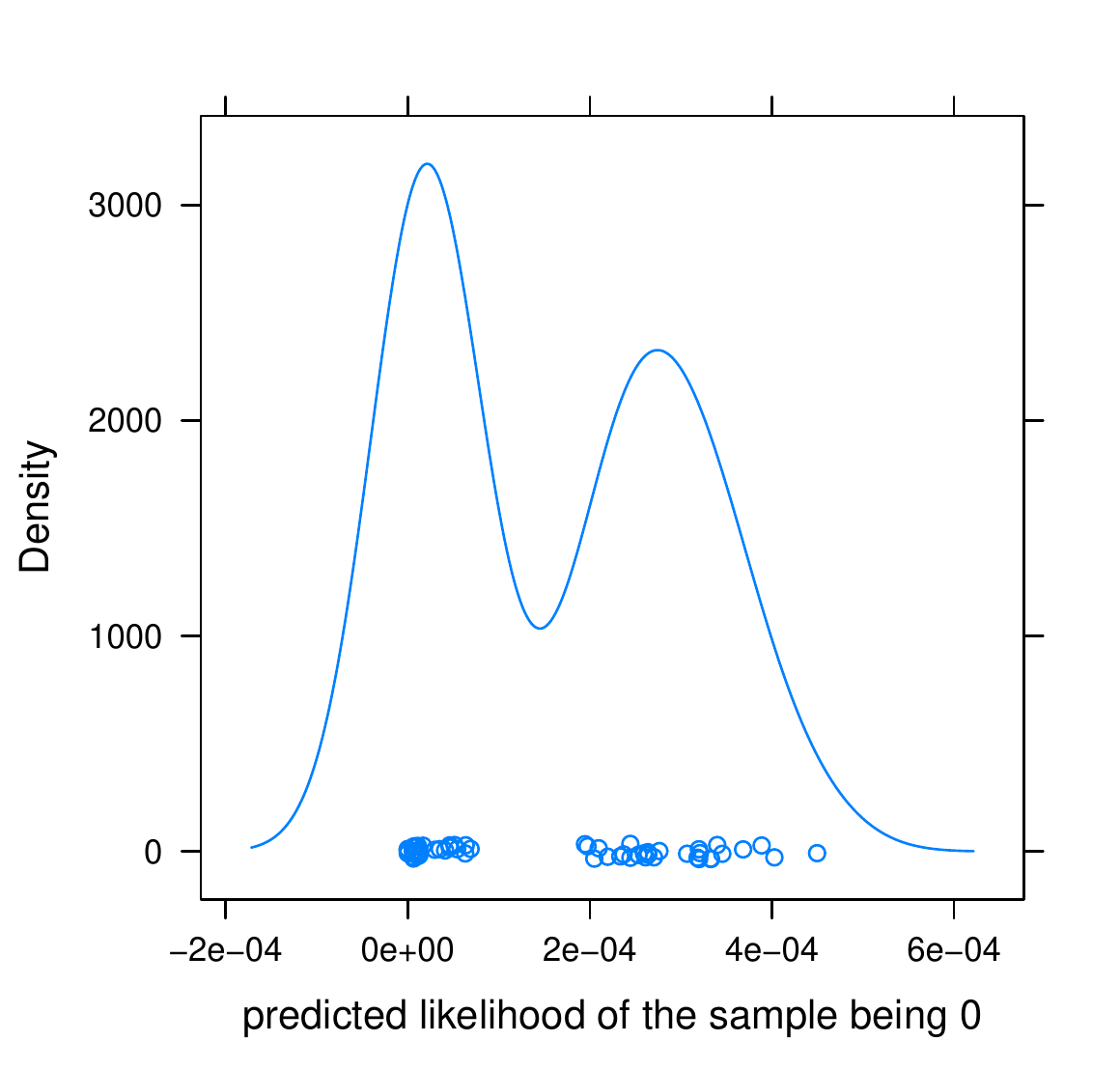}
	\includegraphics[width=0.45\textwidth]{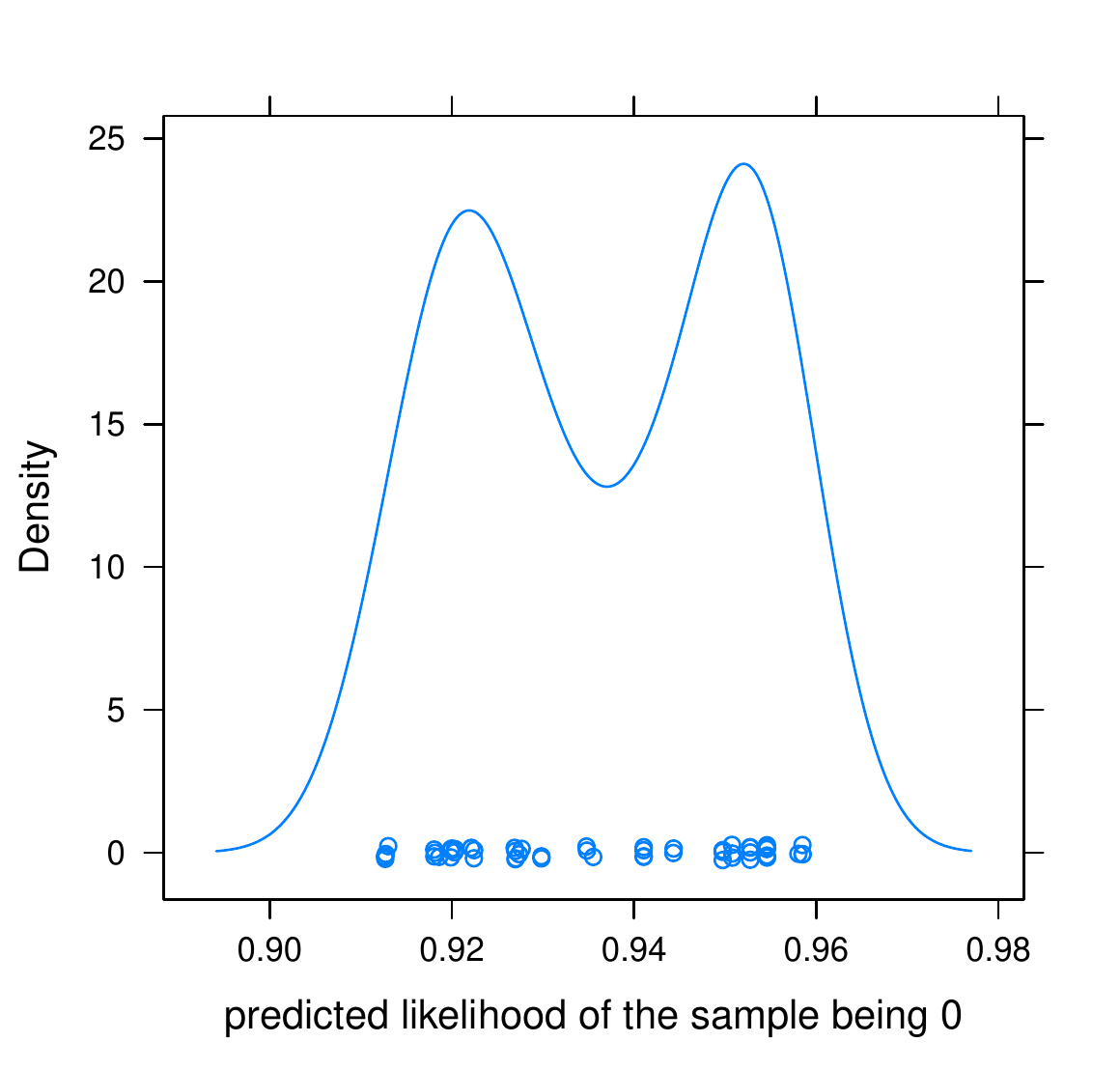}
	\caption{Predicted likelihoods of image \#142 for 100 randomly recombined pairwise-classifiers from networks $A_1$ and $A_2$ using the method of Wu-Lin-Weng. Left image: conditioned on the 0/6 prediction being made by the IIA restriction of the network $A_1$, right image: conditioned on the 0/6 prediction being made by the IIA restriction of the network $A_2$.}
	\label{fig:wu23}
\end{figure}

Thus the method of Wu-Lin-Weng is very sensitive to the information provided by the critical binary classifier, whereas the Bayes covariant method tries to balance information from all binary classifiers. Both have their advantages. The former is more \emph{post-hoc explicable}, since a misprediction is likely caused by a single classifier. On the other hand  the latter is likely to be more precise, since it integrates information from many models.

\section{Sureness explicability}

In this section we examine the problem when a classifier should abstain from making a prediction. In supervised learning this problem is often addressed as follows. Artificial class "other" with samples of images whose prediction should be avoided in original problem is introduced.  Instead of abstaining from making prediction, network is trained to predict this artificial class.
 There is another way to do that in models built with pairwise coupling, and we start by explaining its  theoretical underpinning.

Pairwise coupled models can be viewed as redundant, since any column of the matrix of pairwise likelihoods in theory allows one to infer multi-class posteriors (see Lemma \ref{lemma:column}). In practice however, the derived multi-class likelihoods would not be the same. Geometrically, in ideal situation, when all columns yield the same multi-class prediction, the matrix lies in the image of $\Theta$, the Bradley-Terry manifold. 

Therefore, if a pairwise-coupling classifier system is presented an example which belongs to no known class, we may expect that the resulting matrix of pairwise likelihoods would be much further from the Bradley-Terry manifold compared to the images on which it was trained. 

There is no unique way to measure the distance, but we can easily construct candidate measures. For the method of Wu-Lin-Weng we may take the metric given by optimization criterion $\delta_2$ (see \eqref{eq:wlw2}). For the Bayes covariant method we may take the distance of the reparametrization of the pairwise likelihood matrix given by maps $\theta_{ij}$ (see \eqref{eq:theta}) from its projection onto the Bradley-Terry linear subspace.

We evaluate these metrics on the case, when our $A$-type model is presented with images from a different dataset (digits from MNIST dataset) that was not used during the training. The results for the Bayes covariant coupling method are shown in Figure \ref{fig:sure:bc} and show very good separation of not-to-be classified images from MNIST database. 

\begin{figure}[!ht]
	\centering
	\includegraphics[width = 0.7\textwidth]{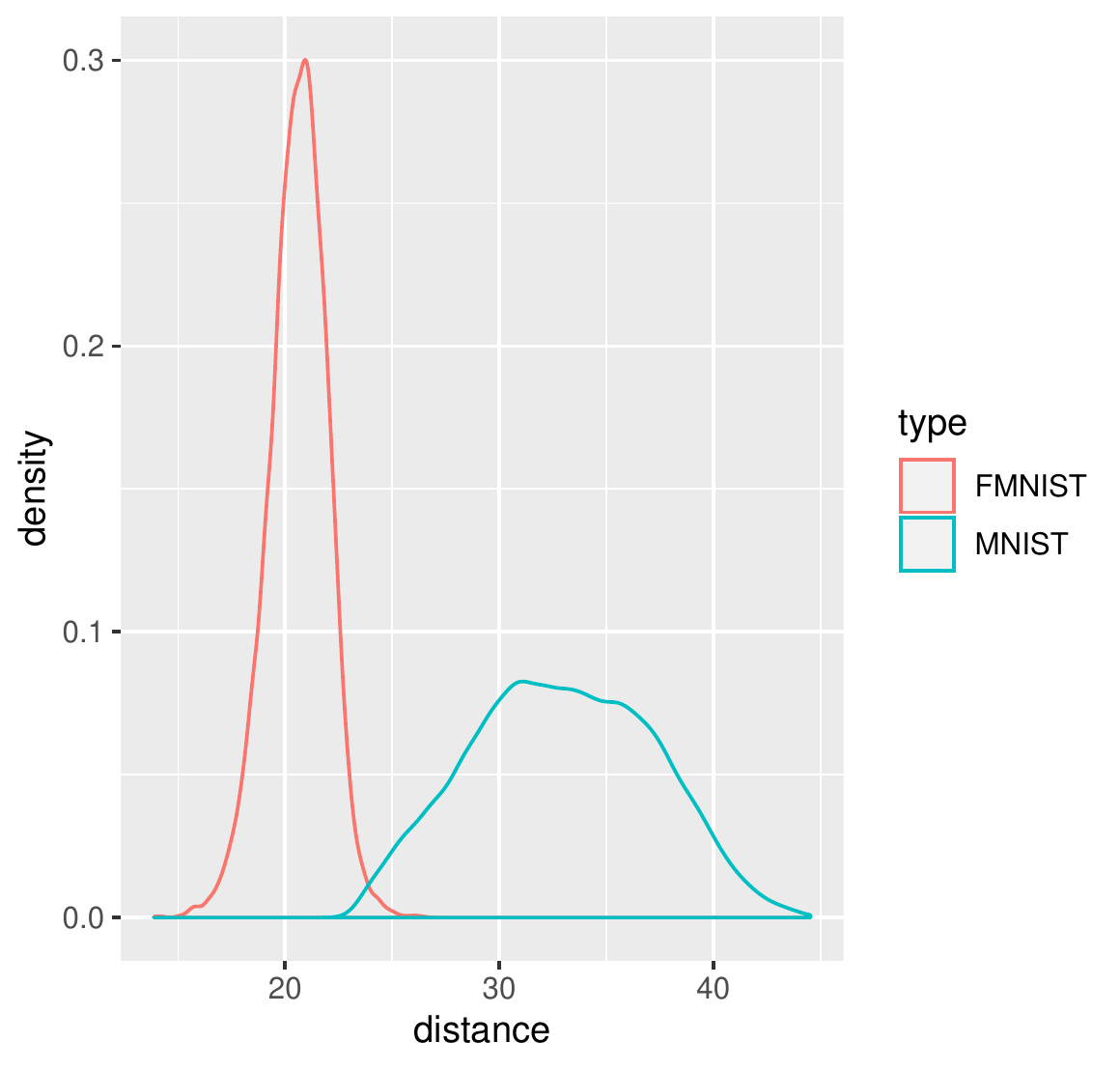}
	\caption{Distribution of distances of pairwise matrices for each testing dataset from the Bradley-Terry manifold for the Bayes-covariant method }
	\label{fig:sure:bc}	
\end{figure}

The results for the method of Wu-Lin-Weng are not as clear-cut. In fact, we had to resort to a diffent kind of visualization in order to analyze them as shown in Figure \ref{fig:sure:wu}. The left plot shows that the distance for Fashion MNIST dataset is close to zero most of the time, as to be expected, but  there are many points which fall in the range typical for MNIST dataset. The right plot reveals the problem - the distances for images from MNIST dataset, which should be large, are rather clustered in two areas - large, and low. The resulting overlap with the distribution of distances for Fashion MNIST may pose a problem in distinguishing among images from Fashion MNIST and those that should be classified as "based on how I've been trained, I don't know".

 \begin{figure}[!ht]
 	\centering
 	\includegraphics[width = 0.4\textwidth]{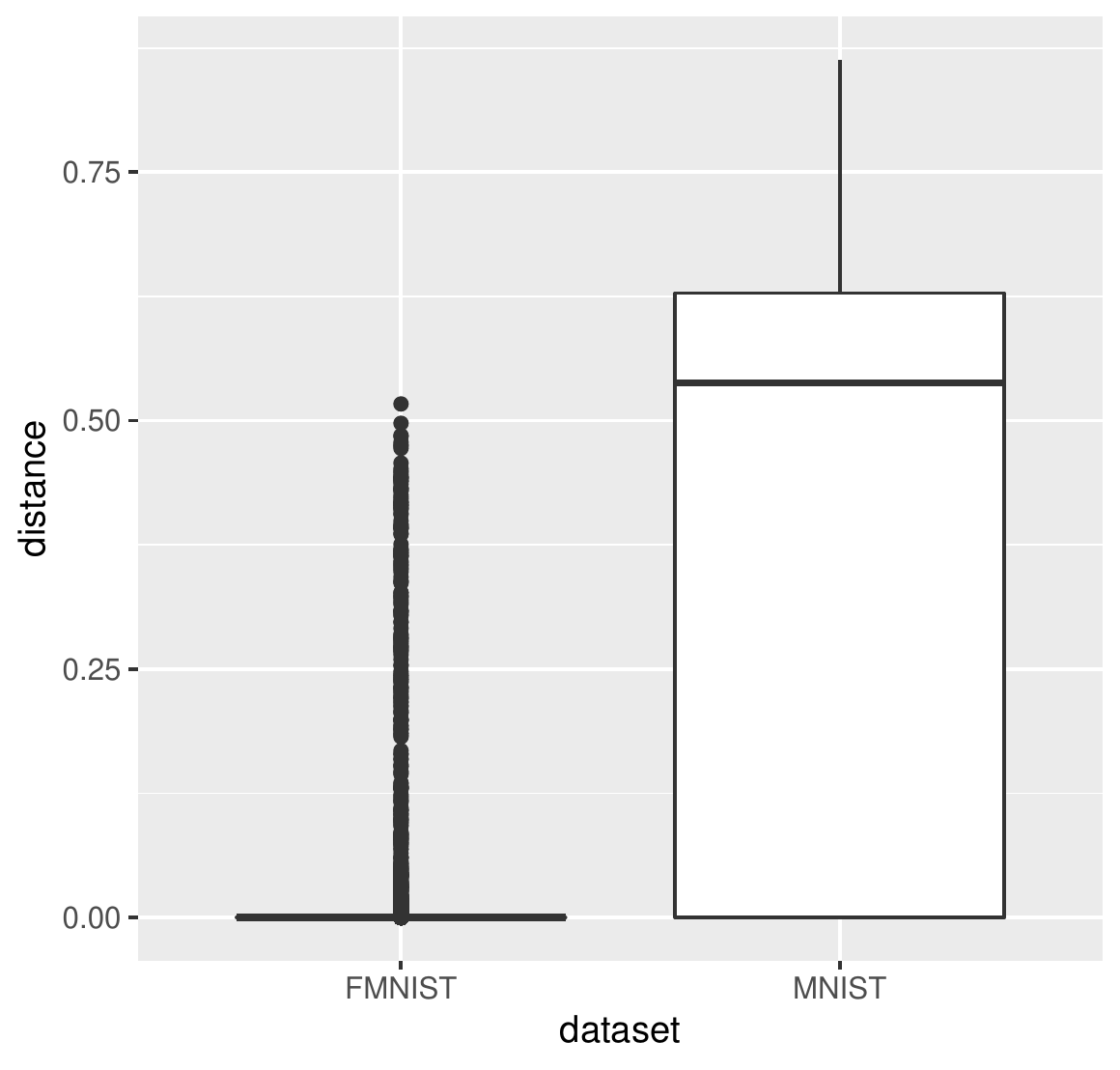}
 	\includegraphics[width = 0.4\textwidth]{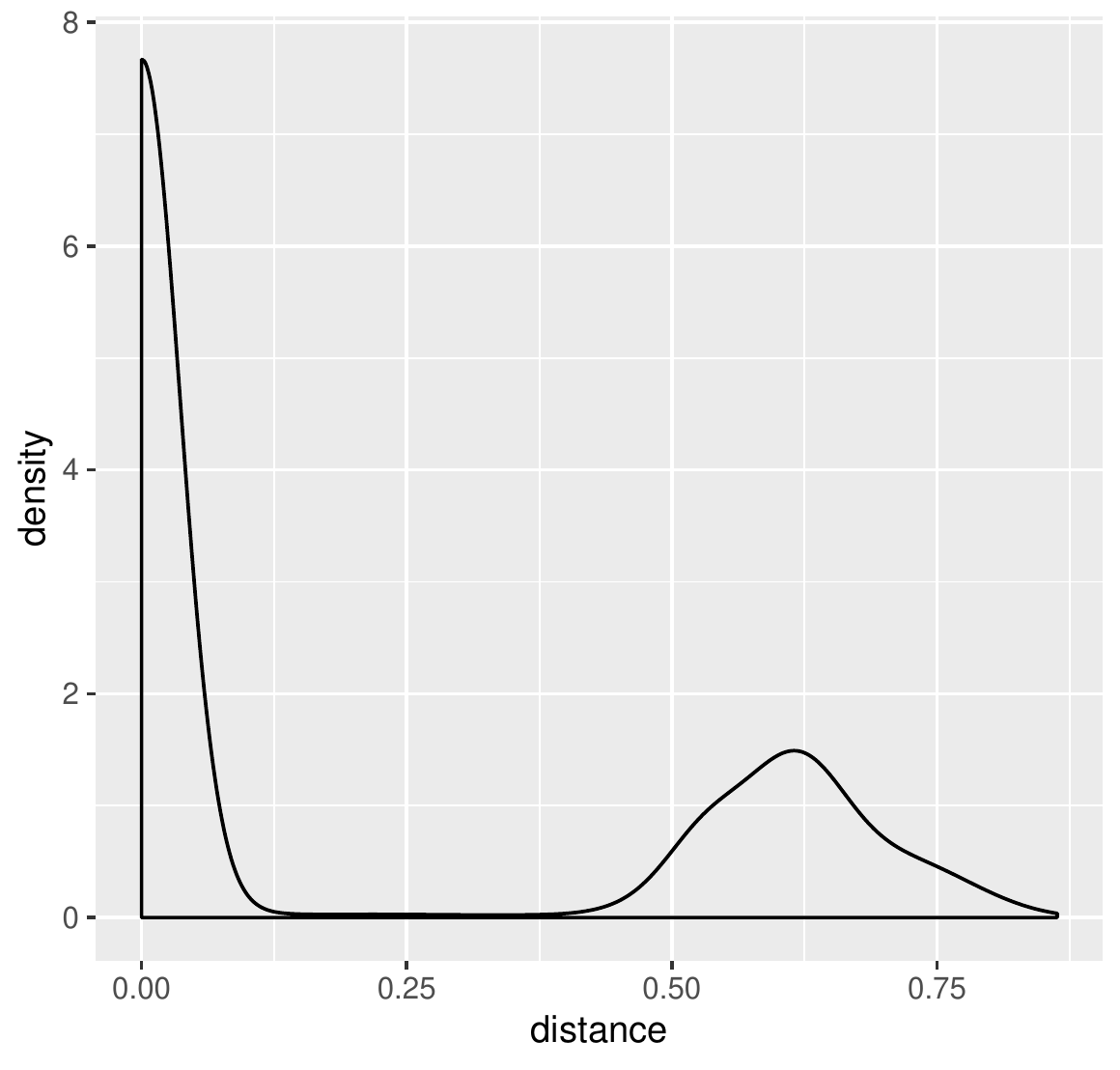}
 	\caption{Results for the Wu-Lin-Weng method. Left: boxplots of distances of pairwise matrices for each testing dataset from the Bradley-Terry manifold. Right: density plot of distances of pairwise matrices for the MNIST dataset from the Bradley-Terry manifold.}
 	\label{fig:sure:wu}	
 \end{figure}

In order to highlight this difficulty, we trained decision trees to classify MNIST from Fashion MNIST based only on the distance using \emph{rpart} package from R. The resulting trees are shown in Figure \ref{fig:trees}, which shows that the decision tree for the Wu-Lin-Weng method is more complex and not natural.

\begin{figure}
	\includegraphics[width = 0.4\textwidth]{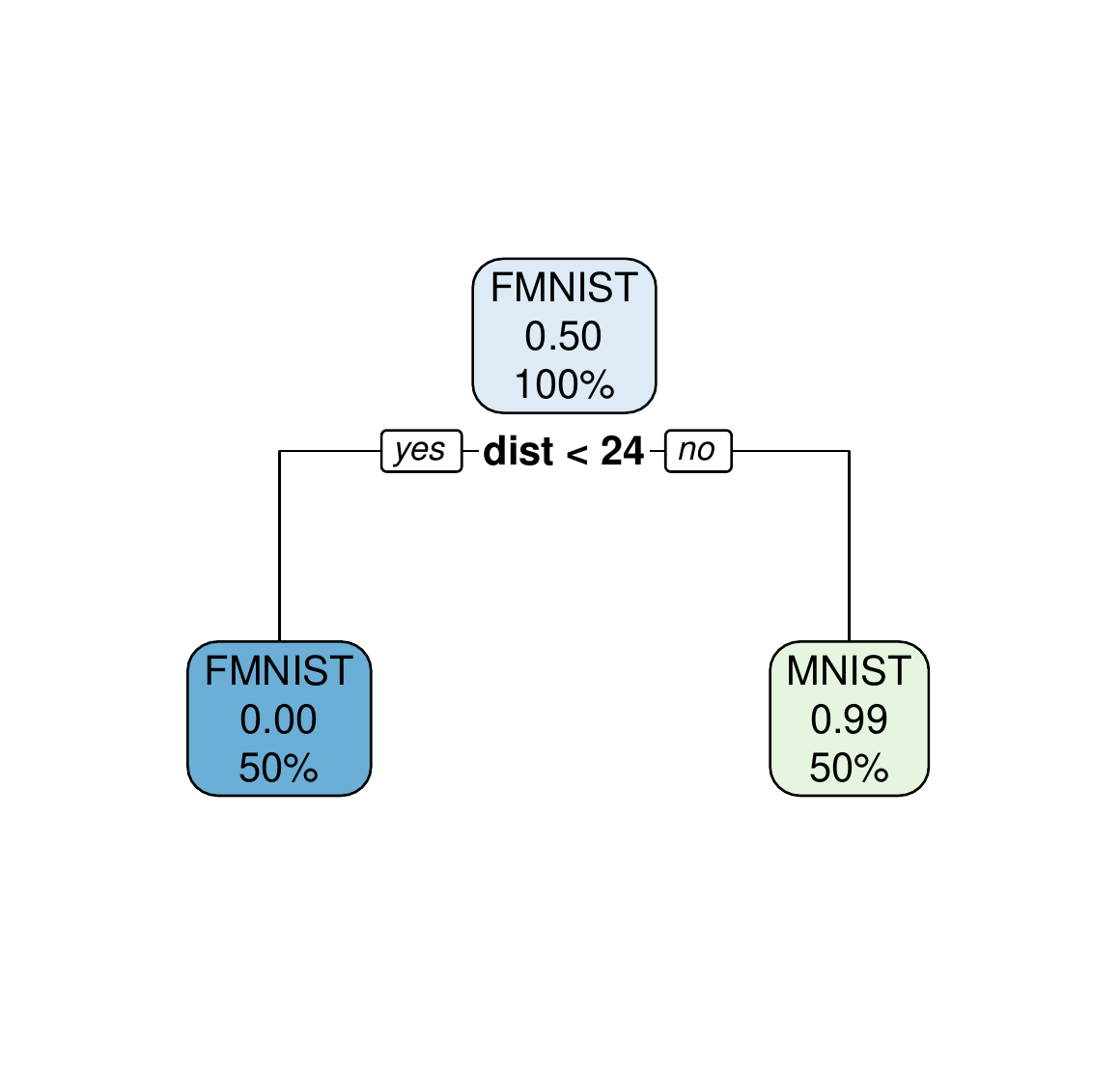}
	\includegraphics[width = 0.4\textwidth]{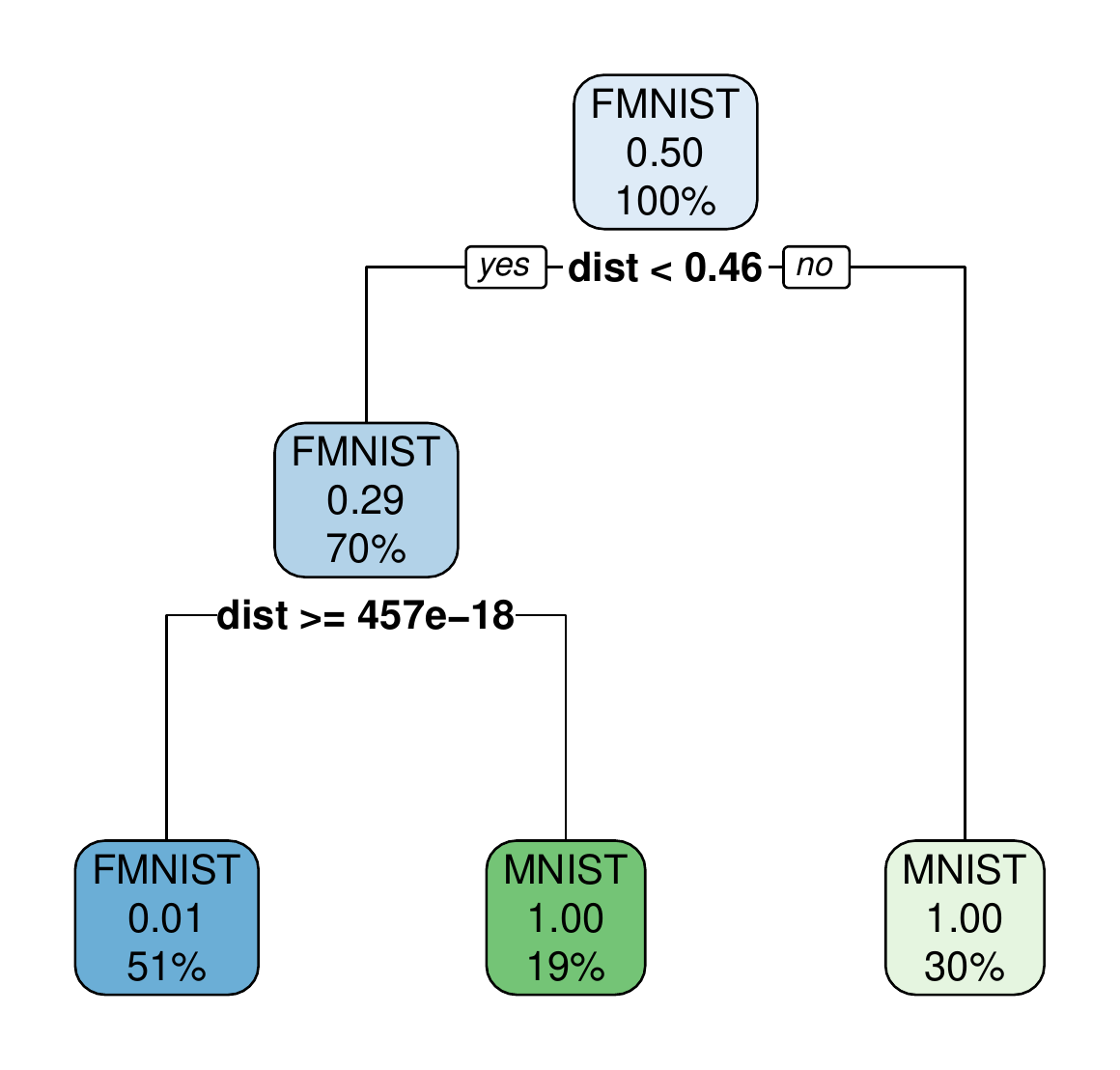}
	
	\caption{Decision trees to distinguish MNIST and Fashion MNIST images using the distance from the Bradley-Terry manifold. Left: using the Bayes covariant method, right: using the method of Wu-Lin-Weng. Note that the left branch corresponds to 'yes' and the right branch to 'no'.}
	\label{fig:trees}
\end{figure}

\section{Conclusion}

\subsection{Discussion}

Let us  discuss the successes and failures discovered via our experiments. 

The first noticeable failure is to obtain better accuracy (pairwise or multi-class) with pairwise coupling methodology using the \emph{same} number of parameters as the baseline convolutional network. We hypothesize that this maybe caused by allocating equal number of parameters to pairwise classification tasks that vary significantly in difficulty (e.g. shirt/t-shirt contrast seems to be by far the most difficult, see Figure \ref{fig:pair-detail} and Table \ref{tab:conf1}). If that is so, then the success of standard multi-class architecture (model $F$) suggests that CNN are able to automatically allocate more capacity to the more challenging tasks. Another possible explanation is that in the presence of multiple classes the convolutional network is able to learn multi-class features that go beyond ones learnable in two-class setting.

However, the subpar pairwise accuracy of micro-networks  and subpar multi-class accuracy of micro-models should be contrasted by the success of mini-networks and mini-models. All mini-networks achieved higher pairwise accuracy that the baseline, and mini-models $A$ and $B$ (although not $C$ and $D$) with the Bayes covariant method (but not the method Wu-Lin-Weng) obtained statistically siginificant improvement in performance over the baseline. Recall that mini-models were designed to have approximately same arithmetical complexity in training to the baseline. We would like to point out that there are likely additional performance advantages to using mini-models:
\begin{itemize}
 	\item the restricted memory size in graphics accelerators favors smaller models, and binary classifiers have less parameters than their multi-class equivalents,
 	\item it is possible to train pairwise models in parallel \emph{without} any communication among the computing nodes,
 	\item models with fewer parameters may require fewer training epochs.
\end{itemize}

However, the accuracy improvements were  modest for both mini-networks and mini-models and by themselves unlikely to attract adoptions. But there are three additional explicability attributes, which are not present in commonly used CNN models.

The first is the ability to incrementally improve a multi-class classification system by incorporating specialized binary classifiers. It is even not mandatory that the specialized binary classifier outperforms the IIA restriction of the previous system. Our results show  that diversity together with a pairwise coupling methodology is able to improve performance even in cases of subpar pairwise performance (cf. Figure \ref{fig:fixes-lm}). Thus we were able to confirm  the phenomenon of \emph{coupling recovery} in real data setting, which has been previously suggested on synthetic data experiments \cite{HT}, \cite{WLW}. 

The second is the ability to gauge randomness in  predicted likelihoods by building many more multi-class systems out of just a couple of pairwise-coupled systems (Section \ref{sec:randomness-explicability}). In this case we started to see a marked difference between the coupling method of Wu-Lin-Weng and the Bayes-covariant coupling method. The former seems to give too much confidence to a single pairwise decision, whereas the latter seems to take into account information from all decisions, leading to much more evenly distributed posteriors. 

The third explicability attribute is the sureness explicability i.e. ability to give probability to the answer "I don't feel confident to make a prediction at all". Due to inherent redundancies of the pairwise coupling model it is possible to derive various measures of sureness. In this case, the Bayes covariant method provided much more satisfactory probabilities, clearly distinguishing cases where it should know the answer, and the cases where it shouldn't feel confident to make a prediction.

We also examined two methodological variations in creating pairwise coupling models. The first was using complementary log-log layer instead of softmax. As seen Figure \ref{fig:anova1} shows, this alternative leads to inferior performance . The second was using non-informative prior like encoding of dependent variable, and this step did not lead to noticeable improvement in classification accuracy. 

\subsection{Open problems}

There is a lot of opportunities to further explore the interaction between convolutional neural networks and pairwise coupling. Our initial results are limited in scope - examining just one dataset, no hyperparameter training nor auto-design of network architecture. Further  evaluations on applied computer vision tasks are obviously desirable. 

Perhaps the most important unexplored issue is how to adapt pairwise coupling methodology to datasets with large number of classes, where training effort and parameter counts become unwieldy. One possibility is to adopt hierarchical classification methodology \cite{hierarchy1}, \cite{hierarchy2}. Another is to devise new coupling methods for incomplete pairwise likelihood matrices.  

Another important open problem is to design a rigorous methodology for designing neural networks for binary classification problems. One may expect that better binary classifiers translate to better multi-class performance in  pairwise-coupled models.

We hope to investigate some of these questions in the future.

\section*{Acknowledgments}

\noindent This work was supported by grant VEGA 2/0144/18. We gratefully acknowledge in-depth discussions with prof. M. Klimo on the subject of sureness explicability.

\section{Bibliography}

\bibliographystyle{elsarticle-num}
\bibliography{my}
\end{document}